\documentclass[11pt,letterpaper,english]{article}

\usepackage[T1]{fontenc}
\usepackage[utf8]{inputenc}
\usepackage[english]{babel}
\usepackage[margin=1in]{geometry}
\usepackage{authblk}
\usepackage{macros}

\usepackage{times}

\newcommand{\expo}[1]{\exp \left( #1 \right)}
\newcommand{\prob}[1]{\mathbb{P} \left[ #1 \right]}
\newcommand{\expect}[1]{\mathbb{E}\left[{#1}\right]}
\newcommand{\indic}[1]{\mathbb{I}\left\{#1\right\}}

\newcommand{\R}{\mathbb{R}}

\allowdisplaybreaks
\begin{document}
\synctex = 1
\title{Concentration bounds for empirical conditional value-at-risk:\\ The unbounded case}

\author[1]{Ravi Kumar Kolla\thanks{ee12d024@ee.iitm.ac.in}}
\author[2]{Prashanth L.A.\thanks{prashla@cse.iitm.ac.in}}
\author[3]{Sanjay P. Bhat\thanks{sanjay.bhat@tcs.com}}
\author[1]{Krishna Jagannathan\thanks{krishnaj@ee.iitm.ac.in}}

\affil[1]{\small Department of Electrical Engineering, Indian Institute of Technology Madras, Chennai, Tamilnadu 600036, India}
\affil[2]{\small Department of Computer Science and Engineering, Indian Institute of Technology Madras, Chennai, Tamilnadu 600036, India}
\affil[3]{\small TCS Research, Hyderabad, Telangana 500081, India}

\renewcommand\Authands{ and }

\date{}
\maketitle
%%%%%%%%%%%%%%%%%%%%%%%%%%%%%%%%%%%%%%%%%%%%%%%%%%%%%%%%%%%%%%
%%%%%%%%%%%%%%%%%%%%%%%%%%%%%%%%%%%%%%%%%%%%%%%%%%%%%%%%%%%%%%
%%%%%%%%%%%%%%%%%%%%%%%%%%%%%%%%%%%%%%%%%%%%%%%%%%%%%%%%%%%%%%
%%%%%%%%%%%%%%%%%%%%%%%%%%%%%%%%%%%%%%%%%%%%%%%%%%%%%%%%%%%%%%

\begin{abstract}
In several real-world applications involving decision making under uncertainty, the traditional expected value objective may not be suitable, as it may be necessary to control losses in the case of a rare but extreme event. Conditional Value-at-Risk (CVaR) is a popular risk measure for modelling the aforementioned objective.
We consider the problem of estimating CVaR from i.i.d. samples of an unbounded random variable, which is either sub-Gaussian or sub-exponential. We derive a novel one-sided concentration bound for a natural sample-based CVaR estimator in this setting. Our bound relies on a concentration result for a quantile-based estimator for Value-at-Risk (VaR), which may be of independent interest. 
%Conditional Value-at-Risk (CVaR) is a popular risk measure for modeling losses in the case of a rare but extreme event. 
%We consider the problem of estimating CVaR from i.i.d. samples of an unbounded random variable, which is either sub-Gaussian or sub-exponential. We derive a novel one-sided concentration bound for a natural sample-based CVaR estimator in this setting. Our bound relies on a concentration result for a quantile-based estimator for Value-at-Risk (VaR), which may be of independent interest. 
\end{abstract}

%\begin{keyword}
%Value-at-risk, Conditional value-at-risk, concentration bounds, sub-Gaussian distributions, sub-exponential distributions.
%\end{keyword}

%%%%%%%%%%%%%%%%%%%%%%%%%%%%%%%%%%%%%%%%%%%%%%%%%%%%%%%%%%%%%%%%%%%%%%%%%%%%%%%
%%%%%%%%%%%%%%%%%%%%%%%%%%%%%%%%%%%%%%%%%%%%%%%%%%%%%%%%%%%%%%%%%%%%%%%%%%%%%%%
%%%%%%%%%%%%%%%%%%%%%%%%%%%%%%%%%%%%%%%%%%%%%%%%%%%%%%%%%%%%%%%%%%%%%%%%%%%%%%%
%%%%%%%%%%%%%%%%%%%%%%%%%%%%%%%%%%%%%%%%%%%%%%%%%%%%%%%%%%%%%%%%%%%%%%%%%%%%%%%

\section{Introduction}
\label{sec:intro}

In several practical decision problems, the presence of uncertainty complicates the decision making process as decisions typically are required to be taken before the uncertainty is resolved. Traditionally, this difficulty is overcome by averaging the costs (or rewards) over all possible realizations of the uncertainty, and then optimizing the averaged cost thus obtained. However, it has been argued that considering averaged outcomes is not appropriate in situations where low-probability events such as financial crashes and category 4 hurricanes can cause huge costs. The possibility of occurrence of such tail events has led to the introduction of risk measures such as Value-at-Risk (VaR) and Conditional Value-at-Risk (CVaR) for quantification of risk. In financial risk management, the VaR of a risky portfolio at a confidence level $\alpha$ is a loss threshold such that the probability of the loss exceeding the threshold is no greater than $1-\alpha$. The CVaR of a portfolio at a confidence level $\alpha$ is the expected loss on the portfolio conditioned on the event that the loss exceeds the VaR. Loosely speaking, VaR quantifies the maximum loss that can occur in the absence of a catastrophic tail event, while the CVaR gives the expected loss given the occurrence of such a tail event. CVaR has several desirable properties as a risk measure. In particular, it is a convex, coherent risk measure (see the survey paper \cite{uryasev} and references therein).  As a result, CVaR continues to receive increasing attention in operations research, mathematical finance, and decision science for problems involving risk quantification or risk minimization. %(CAN WE INCLUDE AN APPROPRIATE BRANCH OF CS HERE IF APPLICABLE?)

In most applications involving uncertainty, the distributions characterizing the underlying uncertain factors are not known, and risk measures such as CVaR have to be estimated from sampled values of the random variable of interest. This is true, for instance, in a multi-armed bandit problem \cite{robbins1952some,bubeck2012regret}
 in which pulling an arm leads to a random loss, and one seeks to identify the arm whose loss random variable has the lowest CVaR by observing a sample of outcomes that result from multiple arm pulls. An obvious estimator for the CVaR of a distribution is the sample CVaR of an i.i.d. sample drawn from the distribution. Naturally, one seeks error bounds for the estimator that help to understand the trade-off between accuracy and sample size. Previous results on CVaR estimation either provide asymptotic error bounds for a general r.v. \cite{sun2010asymptotic}, or provide non-asymptotic error bounds
that hold with high probability, but under the stringent assumption that the underlying r.v. is bounded \cite{brown2007large,wang2010deviation}. 

In this paper, we consider the problem of estimating
the CVaR of an unbounded, albeit sub-Gaussian or sub-exponential random
variable. Sub-Gaussian r.v.s include bounded, Gaussian and any other r.v. whose tail decays as fast as a Gaussian. On the other hand, sub-exponential r.v.s include exponential, Poisson and squared-Gaussian r.v. and are characterized by a tail heavier than Gaussian, resembling that of an exponential distribution. To the best of our knowledge, there are no concentration bounds for CVaR estimator for these two popular classes of unbounded distributions. We believe, imposing a tail decay assumption (sub-Gaussian or sub-exponential) is not restrictive, as such an imposition is common in concentration results for sample mean. Also, the task of CVaR estimation is more challenging in comparison, as it relates to a tail event. We derive a one-sided concentration bound for the empirical CVaR of an i.i.d. sample. Our bound relies on one of the two concentration results (of possibly independent interest) that we provide for a quantile-based estimator for VaR.

The rest of the paper is organized as follows: Section \ref{sec:background} introduces VaR, CVaR and their estimators from i.i.d. samples, Section \ref{sec:results} presents the concentration bounds for VaR and CVaR estimators. Section \ref{sec:proofs} provides detailed proofs of the concentration results, and Section~\ref{sec:conclusions} concludes the paper.

%%%%%%%%%%%%%%%%%%%%%%%%%%%%%%%%%%%%%%%%%%%%%%%%%%%%%%%%%%%%%%%%%%%%%%%%%%%%%%%
%%%%%%%%%%%%%%%%%%%%%%%%%%%%%%%%%%%%%%%%%%%%%%%%%%%%%%%%%%%%%%%%%%%%%%%%%%%%%%%
%%%%%%%%%%%%%%%%%%%%%%%%%%%%%%%%%%%%%%%%%%%%%%%%%%%%%%%%%%%%%%%%%%%%%%%%%%%%%%%
%%%%%%%%%%%%%%%%%%%%%%%%%%%%%%%%%%%%%%%%%%%%%%%%%%%%%%%%%%%%%%%%%%%%%%%%%%%%%%%
\section{Background}
\label{sec:background}
%We begin this section by presenting the definitions of Value-at-Risk (VaR) and Conditional-Value-at-Risk (CVaR) at level $\alpha \in (0, 1)$ for a random variable (r.v.) $X$ with Cumulative Distribution Function (CDF) $F(\cdot).$ Given a r.v. $X$ and $\alpha,$ we use $v_\alpha(X)$ and $c_\alpha(X)$ to denote VaR and CVaR at level $\alpha$ of $X$ respectively~\footnote{For ease of notation, we will not use $X$ in $v_\alpha(X)$ and $c_\alpha(X)$ notations whenever the underlying the r.v. can be understood from the context.}. Typical values of $\alpha$ are 0.95, 0.99 etc.

Given a r.v. $X$ with cumulative distribution function (CDF) $F(\cdot)$, the VaR $v_\alpha(X)$ and CVaR $c_\alpha(X)$~\footnote{For notational brevity, we omit $X$ from $v_\alpha(X)$ and $c_\alpha(X)$ whenever the r.v. can be understood from the context.} at level $\alpha\in (0,1)$ are defined as follows:
\begin{align}
\label{eq:var-cvar-def}
v_\alpha(X) &= \inf \lbrace \xi : \prob{X \leq \xi} \geq \alpha \rbrace \textrm{ and } \\
c_{\alpha}(X)  &=  v_{\alpha}(X)  + \frac{1}{1 - \alpha} \mathbb{E} \left[ X - v_{\alpha}(X) \right] ^+,
\end{align}
where we have used the notation that $[ x ]^+ = \max (0, x)$ for a real number $x.$ Typical values of $\alpha$ chosen in practice are $0.95$ and $0.99$. Note that, if $X$ has a continuous and strictly increasing CDF, then $v_\alpha(X)$ is a solution to the following $\prob{X \leq \xi} = \alpha$ \emph{i.e.,} $v_\alpha(X) = F^{-1}(\alpha).$ CVaR also admits another form under the following assumption:

\noindent\textbf{(A1)} %Assume that the r.v. $X$ has a positive and continuously differentiable density in a small neighbourhood, say $\left[v_\alpha - \eta, v_\alpha + \eta\right]$, of VaR $v_\alpha$. 
The r.v. $X$ is continuous and has strictly increasing CDF.

If (A1) holds and $X$ has a positive density at $v_\alpha,$ then $c_\alpha(X)$ admits the following equivalent form~(cf. \citep{sun2010asymptotic}):
\begin{align*}
c_\alpha(X)  = \expect{X \vert X \geq v_\alpha(X)}.
\end{align*}

Let $\lbrace X_i \rbrace_{i=1}^n$ denote $n$ i.i.d. samples from the distribution of $X$.
Then, the estimates of VaR and CVaR at level $\alpha$, denoted by $\hat{v}_{n, \alpha}$ and $\hat{c}_{n, \alpha}$, are formed as follows~\cite{serfling1980approximation}:
\begin{align}
\hat{v}_{n, \alpha} & = \hat{F}_n^{-1} (\alpha) := \inf \lbrace x : \hat{F}_n(x) \geq \alpha \rbrace \label{eq:var-est}\\
\hat{c}_{n, \alpha} & =  \hat{v}_{n, \alpha}  + \frac{1}{n( 1- \alpha)} \sum_{i=1}^n \left( X_i - \hat{v}_{n,	\alpha} \right) ^+, \label{eq:cvar-est}
\end{align}  
where $\hat{F}_n(x) = \frac{1}{n} \sum_{i=1}^n \mathbb{I} \lbrace X_i \leq x \rbrace $ is the empirical distribution function of $X$. Note that, from the order statistics $X_{[1]}, \ldots,X_{[n]}$, the empirical VaR can be computed as follows:
$\hat{v}_{n, \alpha} = X_{\left[ \lceil n\alpha \rceil \right]}.$

%%%%%%%%%%%%%%%%%%%%%%%%%%%%%%%%%%%%%%%%%%%%%%%%%%%%%%%%%%%%%%%%%%%%%%%%%%%%%%%
%%%%%%%%%%%%%%%%%%%%%%%%%%%%%%%%%%%%%%%%%%%%%%%%%%%%%%%%%%%%%%%%%%%%%%%%%%%%%%%
%%%%%%%%%%%%%%%%%%%%%%%%%%%%%%%%%%%%%%%%%%%%%%%%%%%%%%%%%%%%%%%%%%%%%%%%%%%%%%%
%%%%%%%%%%%%%%%%%%%%%%%%%%%%%%%%%%%%%%%%%%%%%%%%%%%%%%%%%%%%%%%%%%%%%%%%%%%%%%%
\section{Concentration bounds}
\label{sec:results}
In this section, we present four concentration bounds. The first two bounds are for the VaR estimator given in~\eqref{eq:var-est} and these bounds do not impose any restrictions on the underlying distribution. The next two concentration results are for the CVaR estimator given in~\eqref{eq:cvar-est}, and for these results, we assume that the underlying distribution is either sub-Gaussian or sub-exponential (see Definitions \ref{def:subgauss}--\ref{def:subexp} below).

In each of the result presented below, the estimates are calculated using $n$ i.i.d. samples $\lbrace X_i \rbrace_{i=1}^n$ drawn from the r.v. $X$ with CDF $F(\cdot),$ and for a given $\alpha \in (0, 1).$
\begin{proposition}\textbf{\textit{(VaR concentration bound)}}
\label{prop:var-conc-general}
Let $\alpha \in (0, 1),$ $n \in \mathbb{N}$ and $s \in \left(0,\frac{1}{2}\right).$ Define $\alpha^- = \alpha - \frac1{2n^{s}}$ and $\alpha^+ = \alpha + \frac1{2n^{s}}$. Further, let $a_n = \hat{F}_n^{-1}(\alpha^-)$ and $b_n=\hat{F}_n^{-1}(\alpha^+)$, where $\hat{F}_n^{-1}(\cdot)$ is defined by (\ref{eq:var-est}).
  Then, 
  \[\prob{v_{\alpha}(X) \in [a_n,b_n]} \ge  \left(1-2\expo{-\frac{n^{1-2s}}{8}}\right).\]
\end{proposition}
Note that, the above concentration bound is free of any distribution dependent parameters. 
%We now present another representation of VaR estimator given in~\eqref{eq:var-est}. 
\begin{proof}
See Section \ref{sec:proof-var1}.
\end{proof}
%\textbf{Assumption 1~\citep{bahadur1966note}.} There exists a $\gamma > 0$ such that the r.v. $X$ has a positive and continuously differentiable density $f(x)$ in the neighbourhood of $v_\alpha$ \emph{i.e.,} $\forall x \in (v_\alpha - \gamma, v_\alpha + \gamma).$

%Under the above assumption, the VaR estimator given in~\eqref{eq:var-est} can be represented as follows~\citep{bahadur1966note}.  
%\begin{align}
%\label{eq:var-asympototic-representation}
%\hat{v}_{n, \alpha} = v_\alpha + \frac{1}{f(v_\alpha)} \left( \alpha - \frac{1}{n} \sum_{i=1}^n \indic{X_i \leq v_\alpha} \right) + R_n, 
%\end{align}
%where $R_n \leq \frac{b (\log n)^{3/4}}{n^{3/4}}$  for some $b > 0.$

%In the following, we present concentration bounds for the VaR and CVaR estimators given in~\eqref{eq:var-est} and~\eqref{eq:cvar-est}. 

The following result presents another concentration bound for the VaR estimator which will have distribution parameters in the bound. However, unlike Proposition \ref{prop:var-conc-general}, the result presented below is symmetric, and more importantly, bounds the estimation error $\vert \hat{v}_{n, \alpha} - v_\alpha \vert$ directly.
\begin{proposition}[\textbf{\textit{VaR concentration bound}}]
\label{prop:var-concentraion-bound1}
Suppose that (A1) holds. For any $\epsilon > 0,$ we have
\begin{align*}
\prob{\vert \hat{v}_{n, \alpha} - v_\alpha \vert \geq \epsilon} \leq 2 \exp \left(  -2nc\epsilon^2   \right), 
\end{align*}
where $c$ is a constant that depends on the value of the density $f$ of the r.v. $X$ in a neighbourhood of VaR.
%with $\delta_\epsilon = \min \lbrace F(v_\alpha + \epsilon) - F(v_\alpha), F(v_\alpha) - F(v_\alpha - \epsilon) \rbrace,$
\end{proposition}
\begin{proof}
See Section \ref{sec:proof-var2}.
\end{proof}
%The assumption (A1) which requires the existence of a continuously differentiable density is common in the analysis of CVaR estimation, cf. \cite{sun2010asymptotic}. For the case of a continuous r.v. $X$, it is easy to see that $\delta_\epsilon = c \epsilon$, where $c$ depends on the value of the density $f$ of the r.v. in the neighbourhood of VaR. 
The bound above implies that to estimate the VaR to an accuracy of $\epsilon$, one would require an order $O\left(1/\epsilon^2\right)$ number of samples. 
Notice that no restrictive assumptions on the tail of the underlying distribution are made in arriving at the concentration bounds for VaR in Propositions \ref{prop:var-conc-general} and \ref{prop:var-concentraion-bound1}.
However, for establishing concentration bounds for the CVaR, which involves conditioning on a tail event, it is necessary to assume that the distribution is not heavy-tailed. In fact, even for the case of estimating the expected value of a r.v., exponential concentration bounds are available under an assumption that restricts the tail to be light (cf. Chapter 2 of \cite{boucheron2013concentration}).

In this paper, we present concentration bounds under two popular assumptions on the tail of a r.v. The first restricts the r.v. to be sub-Gaussian, while the second requires the same to be sub-exponential. These two classes of r.v.s include bounded r.v.s and more importantly, several unbounded r.v.s as well. Sub-Gaussian r.v.s include the Gaussian r.v.s as well as several other r.v.s whose moment generating functions do not exceed that of a Gaussian, while sub-exponential r.v.s include heavier tailed r.v.s. These two notions are made precise in the following definitions. 

%Before presenting the main result, which is a non-asymptotic bound for the CVaR estimator in \eqref{eq:cvar-est}, we define sub-Gaussian r.v.s below.
\begin{definition}
\label{def:subgauss}
A r.v. $X$ with $\expect{X} = \mu < \infty$ is said to be \emph{$\sigma$-sub-Gaussian} if 
\begin{align*}
\expect{ \expo{\lambda X}} \leq \expo{\lambda \mu + \frac{\lambda^2 \sigma^2}{2}}, \quad \forall \lambda \in \R.
\end{align*}
\end{definition}
\begin{definition}
\label{def:subexp}
A r.v. $X$ with mean $\mu < \infty$ is said to be \emph{$(\sigma,b)$-sub-exponential} if 
\begin{align*}
\expect{ \expo{\lambda X}} \leq \expo{\lambda \mu + \frac{\lambda^2 \sigma^2}{2}}, \quad \forall \left|\lambda\right| < \frac{1}{b}.
\end{align*}
\end{definition}
It is worth noting that all sub-Gaussian r.v.s are sub-exponential, but the converse is not true. The following result presents a one-sided concentration bound for the CVaR estimator in~\eqref{eq:cvar-est}, for the case when the underlying r.v. is sub-Gaussian.
\begin{proposition}[\textbf{\textit{CVaR concentration bound: sub-Gaussian case}}]
\label{prop:cvar-subgauss}
Suppose that (A1) holds. Let $\alpha \in (0, 1),$ and $X$ be a $\sigma$-sub-Gaussian r.v. with mean $\mu$. Suppose that $\alpha$ is large enough to ensure $(v_\alpha - \mu) >0$ and the sub-Gaussian parameter $\sigma$ satisfies $\sigma < \sqrt{\frac{\left( v_\alpha - \mu \right)^2}{2\ln \left( 1/(1-\alpha) \right)} }.$ Then, for any $\epsilon >0,$ we have 
\begin{align}
\prob {\hat{c}_{n, \alpha} - c_\alpha > \epsilon}  \le \exp \left(- \frac{n \epsilon (1-\alpha) (v_\alpha - \mu)}{2\sigma^2} \right) & + 2\exp \left( -2n c_1 \epsilon^2 \right)  + 2\exp \left( -2n c_2 \epsilon \right) \nonumber \\
&+ \exp \left( -2n \epsilon (1-\alpha)^2\right),\label{eq:cvar-conc-subgauss-case}
\end{align}
where $c_1$ and $c_2$ are constants that depend on the value of the density $f$ of the r.v. $X$ in a neighbourhood of VaR.
\end{proposition}

\begin{proof}
See Section \ref{sec:proof-cvar}.
\end{proof}

Suppose that the accuracy $\epsilon$ is greater than $1$. Then, it is apparent that the dominant terms on the RHS of \eqref{eq:cvar-conc-subgauss-case} are those involving an exponential with $\epsilon$. Further, $\hat{c}_{n, \alpha} \le c_\alpha + \epsilon$  with probability (w.p.) at least $(1 - \delta)$, when the number of samples $n$ is of the order $O \left( \frac{1}{\epsilon} \ln \left( \frac{1}{\delta} \right) \right)$. On other hand, an order $O \left( \frac{1}{\epsilon^2} \ln \left( \frac{1}{\delta} \right) \right)$ number of samples are enough to ensure that $\hat{v}_{n, \alpha} \le v_\alpha + \epsilon$ w.p. at least $(1-\delta)$. Hence, CVaR estimation requires more samples in comparison to VaR, when $\epsilon > 1$. In the complementary case, i.e., when $\epsilon <1$, both VaR and CVaR estimate can be $\epsilon$-accurate w.p. $(1-\delta)$, if the number of samples is of the order $O(\frac{1}{\epsilon^2} \ln \left( \frac{1}{\delta} \right))$.

Next, we analyse the concentration of the CVaR estimator in \eqref{eq:cvar-est} for the case when the underlying r.v. is sub-exponential.
\begin{proposition}[\textbf{\textit{CVaR concentration bound: sub-exponential case}}]
\label{prop:cvar-subexp}
Suppose that (A1) holds.  Let $X$ be a $(\sigma, b)$-sub-exponential  r.v. with mean $\mu$. Suppose that $\alpha$ is large enough to ensure $v_\alpha - \mu >0$ and the parameter $\sigma$ satisfies  $\sigma < \sqrt{\frac{2 \ln (1-\alpha) + 2 (v_\alpha - \mu)m_b}{ m_b^2}}$ where $m_b = \min \{\frac{v_\alpha - \mu}{\sigma^2}, b' \}$ where $b' < 1/b.$ 
Then for any $\epsilon > 0,$ we have 
\begin{align}
\prob {\hat{c}_{n, \alpha} - c_\alpha > \epsilon}  \le \exp \left(- \frac{n \epsilon (1-\alpha) m_b}{2} \right) + 2\exp \left( -2n c_1 \epsilon^2 \right)  + 2\exp \left( -2n c_2 \epsilon \right) + \exp \left( -2n \epsilon (1-\alpha)^2\right),\label{eq:cvar-conc-subexp-case}
\end{align} 
where $c_1$ and $c_2$ are as in Proposition~\ref{prop:cvar-subgauss}.
\end{proposition}
From the result above, it is apparent that the rate of CVaR concentration for sub-exponential r.v.s matches that of sub-Gaussian ones.
 %, albeit for small enough accuracy, specified through $\epsilon$. We are not able to obtain CVaR concentration bounds for large $\epsilon$ for sub-exponential case and this is an artefact of our proof technique.
\begin{proof}
See Section \ref{sec:proof-cvar-subexp}.
\end{proof}
\section{Proofs}
\label{sec:proofs}
In this section, we present the proofs of the results presented in Section \ref{sec:results}. 
\subsection{Proof of Proposition~\ref{prop:var-conc-general}}
\label{sec:proof-var1}
\begin{proof}
Recall the Dvoretzky-Kiefer-Wolfowitz (DKW) inequality,  which provides a finite-sample bound on the distance between the empirical distribution and the true distribution:
For any $\epsilon>0$, 
\[
\prob{\sup_{x\in \mathbb{R}}|\hat{F}_n(x)-F(x)|>\epsilon } \leq 2 e^{-2n\epsilon^2}.
\]
 Consider the following event 
 \[A = \Big\lbrace \sup_{x\in [a_n,b_n]} \left| \hat{F}_n(x) - F(x) \right| \le \frac1{4n^s} \Big\rbrace,\] 
 with $a_n = \hat{F}_n^{-1}(\alpha^-)$ and $b_n = \hat{F}_n^{-1}(\alpha^+)$ as defined in the theorem statement. By the DKW inequality, we have 
\begin{align}
 \prob{A} \ge 1 - 2 \expo{\frac{-2n}{16n^{2s}}} =  1 - 2 \expo{\frac{-n^{1-2s}}{8}}.\label{eq:probA}
\end{align}
On the event $A$, we have
\begin{align*}
F(a_n) &\overset{(a)}{=} \lim_{a\uparrow a_{n}}F(a) \overset{(b)}{\le} \lim_{a\uparrow a_{n}}\left(\hat{F}_{n}(a)+\frac1{4n^s}\right) \overset{(c)}{\le}\alpha^{-}+\frac1{4n^s}= \alpha -  \frac1{4n^s} < \alpha, \textrm{ and }\\
  F(b_n) &\ge \alpha^+  -  \frac1{4n^s} =\alpha +  \frac1{4n^s} > \alpha,
\end{align*}
where (a) follows from continuity of $F$, (b) follows from the definition of the event $A$, and (c) follows from the definition of $a_{n}$.
Thus, $v_{\alpha} \in [a_n,b_n]$ and the main claim follows from the lower bound on $\prob{A}$ in \eqref{eq:probA}.
\end{proof}

\subsection{Proof of Proposition~\ref{prop:var-concentraion-bound1}}
\label{sec:proof-var2}
\begin{proof}
\begin{align*}
& \prob{\vert \hat{v}_{n, \alpha} - v_\alpha \vert \geq \epsilon} = \prob{\hat{v}_{n, \alpha} \geq v_\alpha + \epsilon} + \prob{\hat{v}_{n, \alpha} \leq v_\alpha  - \epsilon} \\
 & \leq \prob{\hat{F}_{n}(v_\alpha + \epsilon) \leq \alpha} + \prob{\hat{F}_{n} (v_\alpha - \epsilon) \geq \alpha} \\
&  = \prob{F(v_\alpha + \epsilon) - \hat{F}_{n}(v_\alpha + \epsilon) \geq F(v_\alpha + \epsilon) - \alpha} + \prob{\hat{F}_{n} (v_\alpha - \epsilon) - F(v_\alpha - \epsilon) \geq \alpha - F(v_\alpha - \epsilon)} \\
& \overset{(a)}{\leq} \exp \left( -2n \left( F(v_\alpha + \epsilon) - F(v_\alpha) \right)^2 \right)  +\exp \left( -2n \left( F(v_\alpha) - F(v_\alpha - \epsilon)  \right)^2\right) \\
& \leq 2 \exp \left(  -2n\delta_\epsilon^2   \right), 
\end{align*}
where $(a)$ is due to the DKW inequality, and $\delta_\epsilon = \min \lbrace F(v_\alpha + \epsilon) - F(v_\alpha), F(v_\alpha) - F(v_\alpha - \epsilon) \rbrace.$
\\
Given that the density exists, we have \[ F\left(v_\alpha + \eta_1\right) - F\left(v_\alpha - \eta_2\right) = f(\bar v)(\eta_1+\eta_2), \]
for some $\bar v \in \left[v_\alpha - \eta_2, v_\alpha + \eta_1\right]$. 
Using the identity above for the two expressions inside $\delta_{\epsilon}$, we obtain
\[\delta_{\epsilon} = \min\left(f(\bar v_1), f(\bar v_2)\right)\times \epsilon. \]
for some $\bar v_1 \in \left[v_\alpha, v_\alpha  + \epsilon \right]$ and $\bar v_2 \in \left[v_\alpha - \epsilon, v_\alpha\right].$ 
The claim follows.
\end{proof}

\subsection{Proof of Proposition~\ref{prop:cvar-subgauss}}
\label{sec:proof-cvar}
We first prove a more general result without restricting the sub-Gaussian parameter $\sigma$ and the main claim in Proposition \ref{prop:cvar-subgauss} follows in a straightforward fashion.

\begin{proposition}[\textbf{\textit{General CVaR concentration bound: sub-Gaussian case}}]
\label{prop:cvar-concentration-for-sub-gaussian}
Assume (A1).  Let $X$ be a $\sigma$-sub-Gaussian with mean $\mu$. Suppose that $(v_\alpha - \mu) >0$.  Then, for any $\epsilon >0$, we have 
\begin{align}
 \prob {\hat{c}_{n, \alpha} - c_\alpha > \epsilon} & \le \exp \left(- \frac{n \epsilon (1-\alpha) (v_\alpha - \mu)}{2\sigma^2} \right) \left[ \alpha + \exp \left(  - \frac{( v_\alpha - \mu)^2}{2\sigma^2} \right) \right]^n + 2\exp \left( -2n \delta_{\epsilon_1}^2 \right)  \nonumber \\
& + 2\exp \left( -2n \delta_{\epsilon_2}^2 \right) + \exp \left( -2n \epsilon (1-\alpha)^2\right), \label{eq:cvar-conc-subgauss-general}
\end{align}
where $\delta_{\epsilon_1} = \min \lbrace F(v_\alpha + \frac{n(1-\alpha) \epsilon}{8}) - F(v_\alpha), F(v_\alpha) - F(v_\alpha - \frac{n(1-\alpha) \epsilon}{8})  \rbrace$ and $\delta_{\epsilon_2} = \min \lbrace F(v_\alpha + \sqrt{\epsilon}/4) - F(v_\alpha), F(v_\alpha) - F(v_\alpha - \sqrt{\epsilon}/4) \rbrace.$
\end{proposition}

\begin{proof}
First, we bound the estimate $\hat{c}_{n, \alpha}$. Notice that
\begin{align}
\hat{c}_{n, \alpha} &= \hat{v}_{n,\alpha} +\frac{1}{n(1-\alpha)} \sum_{i=1}^n \left( X_i - \hat{v}_{n, \alpha} \right)^+ \nonumber\\
& = v_\alpha + \frac{1}{n(1-\alpha)} \sum_{i=1}^n \left( X_i - v_\alpha \right)^+ + \left( \hat{v}_{n, \alpha} - v_\alpha \right) + \frac{1}{n(1-\alpha)} \sum_{i=1}^n \left[ \left( X_i - \hat{v}_{n, \alpha} \right)^+ - \left( X_i - v_\alpha \right)^+ \right] \label{eq:asd1} 
\end{align}
The last term on the RHS of \eqref{eq:asd1} can be re-written as follows:
\begin{align}
\frac{1}{n(1-\alpha)} \sum_{i=1}^n \left[ \left( X_i - \hat{v}_{n, \alpha} \right)^+ - \left( X_i - v_\alpha \right)^+ \right] & = \frac{1}{n(1-\alpha)} \sum_{i=1}^n \left[ (v_\alpha - \hat{v}_{n, \alpha}) \indic{X_i \geq \hat{v}_{n,\alpha}}\right] \nonumber \\
& \hspace{-1cm} + \frac{1}{n(1-\alpha)} \sum_{i=1}^n (X_i - v_\alpha) \left[ \indic{X_i \geq \hat{v}_{n, \alpha}} - \indic{X_i \geq v_\alpha} \right], \label{eq:asd2} 
\end{align}
and
\begin{align}
\frac{1}{n(1-\alpha)} \sum_{i=1}^n \left[ (v_\alpha - \hat{v}_{n, \alpha}) \indic{X_i \geq \hat{v}_{n,\alpha}}\right] &= \frac{ \hat{v}_{n, \alpha} - v_\alpha}{1-\alpha} \left[ \left[ \hat{F}_n(\hat{v}_{n, \alpha}) - 1 \right] - \frac{1}{n} \sum_{i=1}^n \indic{X_i = \hat{v}_{n, \alpha}} \right] \nonumber\\
&= \frac{ \hat{v}_{n, \alpha} - v_\alpha}{1-\alpha} \left[ \hat{F}_n(\hat{v}_{n, \alpha}) - 1 \right] \textrm{ w.p. } 1.\label{eq:asd3}
\end{align}
The last equality above uses the fact that $\frac{1}{n} \sum_{i=1}^n \indic{X_i = \hat{v}_{n, \alpha}}$ takes values zero w.p. $1$, since $X_i$ is continuous, for each $i$.

Combining~\eqref{eq:asd1}, \eqref{eq:asd2} and \eqref{eq:asd3}, we obtain
\begin{align}
\hat{c}_{n, \alpha} & = v_\alpha + \frac{1}{n} \sum_{i=1}^n \frac{(X_i - v_\alpha)^+}{1-\alpha} + (\hat{v}_{n, \alpha} - v_\alpha) + \frac{\hat{v}_{n, \alpha} -v}{1-\alpha} \left[ \hat{F}_n(\hat{v}_{n, \alpha}) - 1\right] \nonumber  \\ 
& \qquad \qquad \qquad + \frac{1}{n} \sum_{i=1}^n \frac{X_i - v_\alpha}{1-\alpha} \left[ \indic{X_i \geq \hat{v}_{n, \alpha} } - \indic{X_i \geq v_\alpha } \right] \nonumber  \\
& = v_\alpha + \frac{1}{n(1-\alpha)} \sum_{i=1}^n (X_i - v_\alpha)^+ + B_n, \label{eq:new12}
\end{align}
where
\begin{equation*} 
B_n \overset{\Delta}{=} \frac{\hat{v}_{n, \alpha} - v_\alpha}{1-\alpha} \left[ \hat{F}_n(\hat{v}_{n, \alpha}) - \alpha\right] + \frac{1}{n} \sum_{i=1}^n \frac{X_i - v_\alpha}{1-\alpha} \left[ \indic{X_i \geq \hat{v}_{n, \alpha} } - \indic{X_i \geq v_\alpha } \right].
\end{equation*}
From \eqref{eq:new12}, we have
\begin{align}
\hat{c}_{n, \alpha} - c_\alpha &= \left[ v_\alpha +\frac{1}{n} \sum_{i=1}^n \frac{\left( X_i - v_\alpha \right)^+}{1-\alpha} - c_\alpha \right] + B_n \nonumber \\
&= \frac{1}{1-\alpha} \Big[ \frac{1}{n} \sum_{i=1}^n \left( X_i - v_\alpha \right) \indic{X_i \geq v_\alpha} - \expect{(X-v_\alpha) \indic{X \geq v_\alpha}} \Big] + B_n. \label{eq:c-n-a-n-b-n}
\end{align}
For notational convenience, let 
\begin{align*}
Y_i &= \left( X_i - v_\alpha \right) \indic{X_i \geq v_\alpha},\\
Y &= (X-v_\alpha) \indic{X \geq v_\alpha} \textrm{ and }\\
A_n &= \frac{1}{1-\alpha} \left( \frac{1}{n} \sum_{i=1}^n Y_i - \expect{Y} \right).
\end{align*} 
It is easy to see that $Y_i$'s are i.i.d., non-negative, and $\expect{Y_i} = (1-\alpha)(c_\alpha - v_\alpha)$, $\forall i$. 
We now proceed to bound $\prob{ \hat{c}_{n, \alpha} - c_\alpha > \epsilon}$,  using \eqref{eq:c-n-a-n-b-n}, as follows:
\begin{align}
\prob { \hat{c}_{n, \alpha} - c_\alpha > \epsilon} &= \prob{ A_n + B_n  > \epsilon} \leq \prob{  A_n  > \epsilon/2} + \prob{ B_n  > \epsilon/2}. \label{eq:simple1}
\end{align}
For handling $\prob{ A_n > \epsilon/2}$, we bound the moment generating function of r.v. $Y_i$ as follows: 
\begin{align}
\expect{e^{\lambda Y_i/n}} = \expect{e^{\frac{\lambda}{n}\left( X_i - v_\alpha \right) \indic{X_i \geq v_\alpha} }} & = \int_{-\infty}^{v_\alpha} f_X(x) dx + \int_{v_\alpha}^\infty e^{\frac{\lambda}{n}(x - v_\alpha)} f_X(x) dx \nonumber \\
& \leq F_X(v_\alpha) + e^{-\frac{\lambda}{n} v_\alpha}\int_{-\infty}^\infty e^{\frac{\lambda}{n}x} f_X(x) dx \nonumber \\
& \overset{(a)}{\leq} \alpha + e^{-\frac{\lambda}{n} v_\alpha + \frac{\lambda}{n} \mu + \frac{\lambda^2 \sigma^2}{2n^2}}, \label{eq:mgf-bund-for-yi}	
\end{align}
where $(a)$ is due to the sub-Gaussianity of $X_i$. Thus,
\begin{align}		
\prob{A_n > \frac{\epsilon}{2} } = \prob{\frac{1}{n} \sum_{i=1}^n Y_i > \frac{(1-\alpha)\epsilon}{2} + \expect{Y}} & \overset{(b)}{\leq} \frac{\Pi_{i=1}^n \expect{e^{\lambda Y_i/n}}}{e^{ \lambda \left( (1-\alpha)\epsilon/2 + \expect{Y} \right) }} \overset{(c)}{\leq} \frac{\left[ \alpha + e^{-\frac{\lambda}{n} v_\alpha + \frac{\lambda}{n} \mu + \frac{\lambda^2 \sigma^2}{2n^2}} \right]^n}{e^{ \lambda \left( (1-\alpha)\epsilon/2 + \expect{Y} \right) }}, \label{eq:simple2}
\end{align}
where $(b)$ uses Markov's inequality and $(c)$ follows from \eqref{eq:mgf-bund-for-yi}. Notice that \eqref{eq:simple2} holds for any $\lambda > 0$. However, for the bound on the RHS above to be meaningful, we require that $\alpha + e^{-\frac{\lambda}{n} v_\alpha + \frac{\lambda}{n} \mu + \frac{\lambda^2 \sigma^2}{2n^2}} < 1$. Now, maximizing $\frac{\lambda^2 \sigma^2}{2n^2} -\frac{\lambda(v_\alpha - \mu)}{n}.$ over $\lambda$, we obtain $\lambda_* =  \frac{n ( v_\alpha - \mu )}{\sigma^2}$. Substituting the value of $\lambda_*$ in \eqref{eq:simple2},  we obtain 
\begin{align}
\prob {A_n > \epsilon/2} \leq \exp \left(- \frac{n(v_\alpha - \mu)}{\sigma^2} \left[  \frac{(1-\alpha) \epsilon}{2} + \expect{Y} \right] \right) \left[ \alpha + \exp \left(  - \frac{( v_\alpha - \mu)^2}{2\sigma^2} \right) \right]^n. \label{eq:upperbound-on-An}
\end{align}

For handling the $\prob{B_n > \frac{\epsilon}{2}}$ term in \eqref{eq:simple1}, we bound $\vert B_n \vert$ as follows. Using the inequality above, we bound $\vert B_n \vert$ as follows:
\begin{align*}
\vert B_n \vert & \leq \frac{\vert v_\alpha - \hat{v}_{n, \alpha} \vert}{1-\alpha} \vert \alpha - \hat{F}_n (\hat{v}_{n, \alpha}) \vert + \frac{\vert v_\alpha - \hat{v}_{n, \alpha} \vert}{1-\alpha} \vert \hat{F}_n(v_\alpha) - \hat{F}_n (\hat{v}_{n, \alpha}) \vert \\
& = \frac{\vert v_\alpha - \hat{v}_{n, \alpha} \vert}{1-\alpha} \Big[ \vert F(v_\alpha) - \hat{F}_n (\hat{v}_{n, \alpha}) \vert + \vert \hat{F}_n(v_\alpha) - F(v_\alpha) -\hat{F}_n (\hat{v}_{n, \alpha}) + F(v_\alpha) \vert \Big] \\
& \leq \frac{\vert v_\alpha - \hat{v}_{n, \alpha} \vert}{1-\alpha} \Big[ 2 \vert \hat{F}_n (\hat{v}_{n, \alpha}) - F(v_\alpha) \vert + \vert \hat{F}_n(v_\alpha) - F(v_\alpha) \vert \Big].\stepcounter{equation}\tag{\theequation}\label{eq:notsosimple}
\end{align*}
The first inequality above uses the following fact:
\begin{align*}
\Big| \frac{1}{n} \sum_{i=1}^n \frac{X_i - v_\alpha}{1-\alpha} \left[ \indic{X_i \geq \hat{v}_{n, \alpha} } - \indic{X_i \geq v_\alpha } \right] \Big| \leq \frac{1}{1-\alpha} \vert v_\alpha - \hat{v}_{n, \alpha} \vert \vert \hat{F}_n(v_\alpha) - \hat{F}_n(\hat{v}_{n, \alpha}) \vert.
\end{align*}
Using \eqref{eq:notsosimple}, we have
\begin{align*}
\prob{B_n > \frac{\epsilon}{2}} &\leq \prob{\vert B_n \vert > \frac{\epsilon}{2}} \leq \mathbb{P} \Big[ \frac{\vert \hat{v}_{n,\alpha} - v_\alpha \vert}{1-\alpha}  \Big( 2\vert \hat{F}_{n,\alpha} (\hat{v}_{n,\alpha}) - F(v_\alpha) \vert + \vert \hat{F}_n (v_\alpha) - F(v_\alpha) \vert \Big) > \frac{\epsilon}{2} \Big]
\end{align*}
It is easy to see that $\vert \hat{F}_{n,\alpha} (\hat{v}_{n,\alpha}) - F(v_\alpha) \vert \leq \frac{1}{n}.$ Hence,
\begin{align*}
\prob{B_n > \frac{\epsilon}{2}} & \leq \prob{\vert B_n \vert > \frac{\epsilon}{2}}  \leq \mathbb{P} \Big[ \frac{\vert \hat{v}_{n, \alpha} - v_\alpha \vert}{1 - \alpha} \Big( \frac{2}{n} + \vert \hat{F}_n (v_\alpha) - F(v_\alpha) \vert \Big) > \frac{\epsilon}{2} \Big] \\
& \leq \prob{ \frac{2}{n} \frac{1}{(1-\alpha)} \vert \hat{v}_{n,\alpha} - v_\alpha \vert > \frac{\epsilon}{4} } + \prob{\frac{1}{1-\alpha} \vert \hat{v}_{n,\alpha} - v_\alpha \vert \vert \hat{F}_n (v_\alpha) - F(v_\alpha) \vert > \frac{\epsilon}{4}}.
\end{align*}
Let $D_n = \lbrace \vert \hat{v}_{n,\alpha} - v_\alpha \vert \leq \frac{\sqrt{\epsilon}}{4} \rbrace$. Then, we have 
\begin{align*}
\prob{\frac{\vert \hat{v}_{n,\alpha} - v_\alpha \vert \vert \hat{F}_n (v_\alpha) - F(v_\alpha) \vert}{1-\alpha} > \frac{\epsilon}{4}} & \leq \prob{ \frac{\vert \hat{v}_{n,\alpha} - v_\alpha \vert \vert \hat{F}_n (v_\alpha) - F(v_\alpha) \vert}{1-\alpha} > \frac{\epsilon}{4}, D_n} + \prob{D_n^c} \\
& \leq \prob{\vert \hat{F}_n (v_\alpha) - F(v_\alpha) \vert > \sqrt{\epsilon}(1-\alpha)} + \prob{D_n^c} \\
& \overset{(d)}{\leq} \exp \left( -2n \epsilon (1 - \alpha)^2 \right) + \prob{D_n^c},
\end{align*}
where $(d)$ is due to the DKW inequality. Therefore,
\begin{align}
\prob{B_n > \epsilon/2} \leq \prob{\vert \hat{v}_{n,\alpha} - v_\alpha \vert > n (1-\alpha) \epsilon/8} + \prob{\vert \hat{v}_{n,\alpha} - v_\alpha \vert > \sqrt{\epsilon}/4} + \exp \left( -2n \epsilon (1-\alpha)^2 \right). \label{eq:upper-bound-on-Bn}
\end{align}
Using \eqref{eq:simple1}, \eqref{eq:upperbound-on-An} and \eqref{eq:upper-bound-on-Bn}, we obtain
\begin{align*}
\prob {\hat{c}_{n, \alpha} - c_\alpha > \epsilon} & \leq \exp \left(-\frac{n(v_\alpha - \mu)}{\sigma^2} \left[  \frac{(1-\alpha) \epsilon}{2} + \expect{Y} \right] \right) \left[ \alpha + \exp \left(  - \frac{( v_\alpha - \mu)^2}{2\sigma^2} \right) \right]^n \nonumber \\
& + \prob{\vert \hat{v}_{n,\alpha} - v_\alpha \vert > \frac{n (1-\alpha) \epsilon}{8}} + \prob{\vert \hat{v}_{n,\alpha} - v_\alpha \vert \geq \frac{\sqrt{\epsilon}}{4}} + \exp \left( -2n \epsilon (1-\alpha)^2\right) \\
& \overset{(a)}{\leq}  \exp \left(- \frac{n \epsilon (1-\alpha) (v_\alpha - \mu)}{2\sigma^2}  \right) \left[ \alpha + \exp \left(  - \frac{( v_\alpha - \mu)^2}{2\sigma^2} \right) \right]^n +  2\exp \left( -2n \delta_{\epsilon_1}^2) \right) \nonumber \\
&  + 2\exp \left( -2n \delta_{\epsilon_2}^2) \right) + \exp \left( -2n \epsilon (1-\alpha)^2 \right),
\end{align*}
where $\delta_{\epsilon_1},\delta_{\epsilon_2}$ are as defined in the statement of the proposition, and $(a)$ is due to~Proposition~\ref{prop:var-concentraion-bound1} and the fact that $\expect{Y} \geq 0$.
\end{proof}

\paragraph{Proof of Proposition \ref{prop:cvar-subgauss}}
\begin{proof}
For $\sigma < \sqrt{\frac{\left( v_\alpha - \mu \right)^2}{2\ln \left( 1/(1-\alpha) \right)} },$ 
we note that %$\alpha + \exp \left( \mu - v_\alpha + \sigma^2/2 \right) \leq 1.$ 
$\alpha + \exp \left(  - \frac{( v_\alpha - \mu)^2}{2\sigma^2} \right) < 1.$
Given that the density exists, we have
\[ F\left(v_\alpha + \eta_1\right) - F\left(v_\alpha - \eta_2\right) = f(\bar v)(\eta_1+\eta_2),\]
for some $\bar v \in \left[v_\alpha - \eta_2, v_\alpha + \eta_1\right]$. 
Using the identity above for the two expressions inside $\delta_{\epsilon_1}$, we obtain
\[\delta_{\epsilon_1} = \min\left(f(\bar v_1), f(\bar v_2)\right)\times \frac{n(1-\alpha) \epsilon}{8}.\]
for some $\bar v_1 \in \left[v_\alpha, v_\alpha  + \frac{n(1-\alpha) \epsilon}{8}\right]$ and $\bar v_2 \in \left[v_\alpha - \frac{n(1-\alpha) \epsilon}{8}, v_\alpha\right]$. 
Along similar lines, it is easy to infer that 
\[\delta_{\epsilon_2} = \min\left(f(\bar v_3), f(\bar v_4)\right)\times\frac{\sqrt\epsilon}{4},\]
for some $\bar v_1 \in \left[v_\alpha, v_\alpha  + \frac{\sqrt\epsilon}{4}\right]$ and $\bar v_2 \in \left[v_\alpha - \frac{\sqrt\epsilon}{4},v_\alpha\right]$.
The claim follows.
\end{proof} 

\subsection{Proof of Proposition~\ref{prop:cvar-subexp}}
\label{sec:proof-cvar-subexp}
\begin{proof}
Observe that, in the proof of Proposition~\ref{prop:cvar-concentration-for-sub-gaussian},  sub-Gaussianity is used is bounding   $\prob{A_n \geq \epsilon/2}$ following \eqref{eq:simple1} there. Here, we bound the same using sub-exponential assumption.     

Recall that $A_n = \frac{1}{1-\alpha} \left( \frac{1}{n} \sum_{i=1}^n Y_i - \expect{Y} \right),$ where $Y_i = \left( X_i - v_\alpha \right) \indic{X_i \geq v_\alpha}$ and $Y = (X-v_\alpha) \indic{X \geq v_\alpha}.$  Starting as in the derivation of~\eqref{eq:mgf-bund-for-yi},
\begin{align}
\expect{e^{\lambda Y_i}} = \expect{e^{ \lambda\left( X_i - v_\alpha \right) \indic{X_i \geq v_\alpha} }} & \leq F_X(v_\alpha) + e^{-\lambda v_\alpha}\int_{-\infty}^\infty e^{\lambda x} f_X(x) dx \nonumber \\
& \leq \alpha + e^{-\lambda v_\alpha + \lambda \mu +  \frac{\lambda^2 \sigma^2}{2}}, \,\,\,\,\, \forall | \lambda | < \frac{1}{b}, \label{eq:mgf1-bund-for-yi}	
\end{align}
where the last inequality uses the fact that $X_i$ is sub-exponential, for each $i$. Thus,
\begin{align}		
\prob{A_n > \frac{\epsilon}{2} } = \prob{ \sum_{i=1}^n Y_i > n \left( \frac{(1-\alpha)\epsilon}{2} + \expect{Y} \right)} & \overset{(b)}{\leq} \frac{\Pi_{i=1}^n \expect{e^{\lambda Y_i}}}{e^{ \lambda \left( n (1-\alpha)\epsilon/2 + \expect{Y} \right) }}  \quad \forall \lambda > 0  \nonumber\\
& \overset{(c)}{\leq} \frac{\left[ \alpha + e^{-\lambda v_\alpha + \lambda \mu + \frac{\lambda^2 \sigma^2}{2}} \right]^n}{e^{ n \lambda \left( (1-\alpha)\epsilon/2 + \expect{Y} \right) }}, \quad \forall \,  0 < \lambda < \frac{1}{b}, \label{eq:ssimple2}
\end{align}
where $(b)$ follows from Markov's inequality and $(c)$ uses \eqref{eq:mgf1-bund-for-yi}. As argued in the proof of Proposition~\ref{prop:cvar-concentration-for-sub-gaussian}, choosing $\lambda = \frac{v_\alpha - \mu}{\sigma^2}$ maximizes the term in the numerator on the RHS above, while ensuring the same is less than $1$. However, sub-exponential assumption requires that $0 < \lambda < 1/b.$ Hence, we choose $\lambda = \left\{m_b = \min\left(\frac{v_\alpha - \mu}{\sigma^2}, b'\right)\right\},$ where $b' <  1/b.$ Using the above and the fact that $\expect{Y} \geq 0,$ we obtain that
\begin{align}
\prob {A_n > \epsilon/2}  \leq \exp \left(- \frac{n \epsilon   (1-\alpha) m_b }{2} \right)  \left[ \alpha + \exp \left( m_b (\mu - v_\alpha) + \frac{m_b^2\sigma^2}{2} \right) \right]^n. \label{eq:upperbound1-on-An}
\end{align}
For $\sigma < \sqrt{\frac{2 \ln (1-\alpha) + 2 (v_\alpha - \mu)m_b}{ m_b^2}},$ we obtain
\begin{align*}
\prob {A_n > \epsilon/2}  \leq \exp \left(- \frac{n \epsilon (1-\alpha) m_b}{2} \right)
\end{align*}
The rest of the proof follows in a similar manner as that of Proposition~\ref{prop:cvar-subgauss}.
\end{proof} 

\begin{remark}
The proof technique used to establish Proposition \ref{prop:cvar-subgauss} cannot be employed to establish a lower deviations bound for the CVaR estimator in \eqref{eq:cvar-est}. 
In particular, using the notation from the proof of Proposition \ref{prop:cvar-subgauss}, we have
\begin{align*}
\prob{\hat{c}_{n, \alpha} - c_\alpha \leq -\epsilon} = \prob{A_n + B_n \leq -\epsilon} & \leq \prob{A_n \leq -\epsilon/2} + \prob{B_n \leq -\epsilon/2} \\
& \leq \prob{A_n \leq -\epsilon/2} + \prob{\vert B_n \vert \geq \epsilon/2}
\end{align*}
While $\prob{\vert B_n \vert \leq -\epsilon/2}$ can be bounded as before (see~\eqref{eq:upper-bound-on-Bn}), handling  $\prob{A_n \leq -\epsilon/2}$ is challenging. For instance, mimicking the steps leading to \eqref{eq:simple1}, we have
\begin{align*}
\prob{A_n \leq -\epsilon/2} &= \prob{\frac{1}{n(1-\alpha)} \sum_{i=1}^n \left( Y_i - \expect{Y} \right) \leq -\epsilon/2} \\
& = \prob{\frac{1}{n} \sum_{i=1}^n  Y_i \leq \expect{Y} -(1 - \alpha) \epsilon/2} \\
& \overset{(a)}{=} \prob{ \exp \left(- \frac{\lambda}{n} \sum_{i=1}^n  Y_i \right)\right. \left. \geq \exp \left( -\lambda \left( \expect{Y} - \frac{(1 - \alpha) \epsilon}{2} \right) \right)}  \\
& \overset{(b)}{\leq} \frac{\Pi_{i=1}^n \expect{e^{\frac{-\lambda Y_i}{n}  }}}{e^{ -\lambda \left( \expect{Y} -(1 - \alpha)\epsilon/2  \right)}}\\
&  \overset{(c)}{\leq} \frac{\left[ \alpha + e^{\frac{\lambda}{n} v_\alpha - \frac{\lambda}{n} \mu + \frac{\lambda^2 \sigma^2}{2n^2}} \right]^n}{\exp\left( -\lambda \left( -(1-\alpha)\epsilon/2 + \expect{Y} \right) \right)},\stepcounter{equation}\tag{\theequation}\label{eq:sdt}
\end{align*}
where $(a)$ holds for any $\lambda > 0,$ $(b)$ uses Markov's inequality and $(c)$ is due to~\eqref{eq:mgf-bund-for-yi}. Deriving a meaningful upper bound on $\prob{A_n \leq -\epsilon/2}$ using \eqref{eq:sdt} is difficult, since the numerator there cannot be controlled. This is because $(v_\alpha-\mu)>0$ and $\lambda$ is constrained to be positive. 
\end{remark}
%\subsection{$Y$ is sub-Gaussian claim}
%Suppose $X$ is $\sigma$-sub-Gaussian. Let $\bar X = X - v_\alpha$ with mean $\bar \mu = \mu - v_\alpha$. Then, 
%\begin{align*}
%\expect{ \expo{\lambda \bar X}} \leq \expo{\lambda \bar\mu + \frac{\lambda^2 \sigma^2}{2}}, \qquad \forall \lambda \in \R.
%\end{align*}
%Using Chernoff method, we have
%\begin{align}
% \prob{\left|\bar X - \bar \mu\right| > \epsilon} \le   \exp \left(- \frac{\epsilon^2}{2 \sigma^2}\right).
%\end{align}
%Let $\bar Y = \bar X \indic{X \geq v_\alpha}$. 
%Then, $\left| \bar Y \right| \le \left| \bar X \right|$ leads to
%\begin{align}
% \prob{\left|\bar Y - \mu_{\bar Y} \right| > \epsilon} \le \prob{\left|\bar X - \bar \mu\right| > \epsilon}
%\end{align}

%%%%%%%%%%%%%%%%%%%%%%%%%%%%%%%%%%%%%%%%%%%%%%%%%%%%%%%%%%%%%%%%%%%%%%%%%%%%%%%
%%%%%%%%%%%%%%%%%%%%%%%%%%%%%%%%%%%%%%%%%%%%%%%%%%%%%%%%%%%%%%%%%%%%%%%%%%%%%%%
%%%%%%%%%%%%%%%%%%%%%%%%%%%%%%%%%%%%%%%%%%%%%%%%%%%%%%%%%%%%%%%%%%%%%%%%%%%%%%%
%%%%%%%%%%%%%%%%%%%%%%%%%%%%%%%%%%%%%%%%%%%%%%%%%%%%%%%%%%%%%%%%%%%%%%%%%%%%%%%

%%%%%%%%%%%%%%%%%%%%%%%%%%%%%%%%%%%%%%%%%%%%%%%%%%%%%%%%%%%%%%%%%%%%%%%%%%%%%%%
%%%%%%%%%%%%%%%%%%%%%%%%%%%%%%%%%%%%%%%%%%%%%%%%%%%%%%%%%%%%%%%%%%%%%%%%%%%%%%%
%%%%%%%%%%%%%%%%%%%%%%%%%%%%%%%%%%%%%%%%%%%%%%%%%%%%%%%%%%%%%%%%%%%%%%%%%%%%%%%
%%%%%%%%%%%%%%%%%%%%%%%%%%%%%%%%%%%%%%%%%%%%%%%%%%%%%%%%%%%%%%%%%%%%%%%%%%%%%%%

\section{Conclusions}
\label{sec:conclusions}
We derived a one-sided concentration bound for a natural sample-based CVaR estimator, when the underlying distribution is unbounded, albeit sub-Gaussian or sub-exponential. We also derived concentration bounds for a quantile-based estimator for VaR, and this may be of independent interest. We believe our concentration bounds for natural estimates VaR and CVaR are interesting not only from a statistical viewpoint, but also for solving sequential decision making problems under uncertainty, for e.g., in the multi-armed bandit framework \cite{bubeck2012regret}. 

As future work, it would be interesting to derive a concentration result that bounds the lower deviations of the CVaR estimator. An orthogonal direction is to derive concentration result for CVaR estimators that incorporate importance sampling. 

%just say we derived a one-sided concentration result for unbounded r.v.s and it would be interesting future work to derive the other side as welland may be do a bandit application

%%%%%%%%%%%%%%%%%%%%%%%%%%%%%%%%%%%%%%%%%%%%%%%%%%%%%%%%%%%%%%%%%%%%%%%%%%%%%%%
%%%%%%%%%%%%%%%%%%%%%%%%%%%%%%%%%%%%%%%%%%%%%%%%%%%%%%%%%%%%%%%%%%%%%%%%%%%%%%%
%%%%%%%%%%%%%%%%%%%%%%%%%%%%%%%%%%%%%%%%%%%%%%%%%%%%%%%%%%%%%%%%%%%%%%%%%%%%%%%
%%%%%%%%%%%%%%%%%%%%%%%%%%%%%%%%%%%%%%%%%%%%%%%%%%%%%%%%%%%%%%%%%%%%%%%%%%%%%%%
\bibliographystyle{plainnat}
\bibliography{cvar_refs}

\end{document}